\documentclass{article}

\usepackage[hyphens]{url}            

\usepackage[final]{neurips_2019}

\usepackage[utf8]{inputenc} 
\usepackage[T1]{fontenc}    
\usepackage{hyperref}       
\usepackage{booktabs}       
\usepackage{amsfonts}       
\usepackage{nicefrac}       
\usepackage{microtype}      

\usepackage{amsmath}
\usepackage{amsthm}
\usepackage{enumitem}
\usepackage{algorithm}
\usepackage[noend]{algorithmic}
\usepackage{graphicx}
\usepackage{subcaption}
\usepackage{float}
\usepackage{multicol}
\usepackage{wrapfig}
\usepackage{color} 
\usepackage{tikz}
\usetikzlibrary{arrows,shapes}

\newtheorem{theorem}{Theorem}[section]
\newtheorem{definition}[theorem]{Definition}
\newtheorem{lemma}[theorem]{Lemma}
\newtheorem{proposition}[theorem]{Proposition}

\newcommand\shrink[1]{}

\def\n(#1){\bar{#1}}

\def\W{{\bf W}}

\def\X{{\bf X}}
\def\x{{\bf x}}

\def\eql(#1,#2){{#1\!\!=\!#2}}

\def\eql(#1,#2){{#1\!=\!#2}}
\newcommand\name[1]{\ensuremath{\mathsf{#1}}}

\def\clap#1{\hbox to 0pt{\hss#1\hss}}

\newcommand\Ainput[1]{
\vspace{1mm}
\textbf{input:} #1
}

\newcommand\Aoutput[1]{
\vspace{1mm}
\textbf{output:} #1
\vspace{1mm}
}

\makeatletter
\def\thm@space@setup{%
  \thm@preskip=0.2cm plus 0cm minus 0cm
  \thm@postskip=0cm plus 0cm minus 0cm
}
\renewenvironment{proof}[1][\proofname]{\par
  \pushQED{\qed}%
  \normalfont
  \topsep0pt \partopsep0pt 
  \trivlist
  \item[\hskip\labelsep
        \itshape
    #1\@addpunct{.}]\ignorespaces
}{%
  \popQED\endtrivlist\@endpefalse
}
\makeatother

\renewcommand{\algorithmiccomment}[1]{\bgroup\hfill//~#1\egroup}

\title{Smoothing Structured Decomposable Circuits}

\author{%
  Andy Shih\\
  University of California, Los Angeles\\
  \texttt{andyshih@cs.ucla.edu} \\
  \And
  Guy Van den Broeck\\
  University of California, Los Angeles\\
  \texttt{guyvdb@cs.ucla.edu} \\
  \And
  \phantom{0000000}Paul Beame\\
  \phantom{0000000}University of Washington\\
  \texttt{\phantom{0000000}beame@cs.washington.edu} \\
  \And
  \phantom{00}Antoine Amarilli\\
  \phantom{00}LTCI, Télécom Paris, IP Paris\\
  \texttt{\phantom{00}antoine.amarilli@telecom-paris.fr} \\
}

\begin{document}

\maketitle

\begin{abstract}

We study the task of \emph{smoothing} a circuit, i.e., ensuring that all children of a \(\oplus\)-gate mention the same variables. 
Circuits serve as the building blocks of state-of-the-art inference algorithms on discrete probabilistic graphical models and probabilistic programs. They are also important for discrete density estimation algorithms.
Many of these tasks require the input circuit to be smooth. However, smoothing has not been studied in its own right yet, and only a trivial quadratic algorithm is known. This paper studies efficient smoothing for structured decomposable circuits. We propose a near-linear time algorithm for this task and explore lower bounds for smoothing decomposable circuits, using existing results on range-sum queries. Further, for the important case of All-Marginals, we show a more efficient linear-time algorithm. We validate experimentally the performance of our methods.
\end{abstract}

\section{Introduction}
Circuits are directed acyclic graphs that are used for many logical and probabilistic inference tasks. Their structure captures the computation of reasoning algorithms. In the context of machine learning, state-of-the-art algorithms for exact and approximate inference in discrete probabilistic graphical models~\citep{Chavira2008OnPI,Kisa2014ProbabilisticSD,FriedmanNeurIPS18} and probabilistic programs~\citep{Fierens2015InferenceAL,Bellodi2013ExpectationMO} are built on circuit compilation. In addition, learning tractable circuits is the current method of choice for discrete density estimation~\citep{Gens2013LearningTS,Rooshenas2014LearningSN,Vergari2015SimplifyingRA,Liang2017LearningTS}. Circuits are also used to enforce logical constraints on deep neural networks~\citep{Xu2018ASL}.

Most of the probabilistic inference algorithms on circuits actually require the input circuit to be \emph{smooth} (also referred to as \emph{complete})~\citep{Sang2005PerformingBI,Poon2011SumproductNA}. The notion of smoothness was first introduced by~\cite{Darwiche2001OnTT} to ensure efficient model counting and cardinality minimization and has since been identified as essential to probabilistic inference algorithms. Yet, to the best of our knowledge, no efficient algorithm to smooth a circuit has been proposed beyond the original quadratic algorithm by~\cite{Darwiche2001OnTT}.

The quadratic complexity can be a major bottleneck, since circuits in practice often have hundreds of thousands of edges when learned, and millions of edges when compiled from graphical models. As such, in the latest Dagstuhl Seminar on ``Recent Trends in Knowledge Compilation'', this task of smoothing a circuit was identified as a major research challenge~\citep{Darwiche2017RecentTI}. Therefore, a more efficient smoothing algorithm will increase the scalability of circuit-based inference algorithms.

Intuitively, smoothing a circuit amounts to filling in the missing variables under its \(\oplus\)-gates. In Figure~\ref{fig:unsmoothed} we see that the \(\oplus\)-gate does not mention the same variables on its left side and right side, so we fill in the missing variables by adding tautological gates of the form \(x_i \oplus\ -x_i\), resulting in the smooth circuit in Figure~\ref{fig:smoothed}. Filling in these missing variables is necessary for probabilistic inference tasks such as computing marginals, computing probability of evidence, sampling, and approximating Maximum A Posteriori inference~\citep{Sang2005PerformingBI,Chavira2008OnPI,Friesen2016TheST,FriedmanNeurIPS18, Mei2018MaximumAP}. The task of smoothing was also explored by~\cite{Peharz2017OnTL}, where they look into preserving smoothness when augmenting Sum-Product Networks for computing Most Probable Explanations.

\begin{figure}[t]
    \centering
    \begin{subfigure}[b]{0.15\linewidth}
        \centering
        \includegraphics[width=\linewidth]{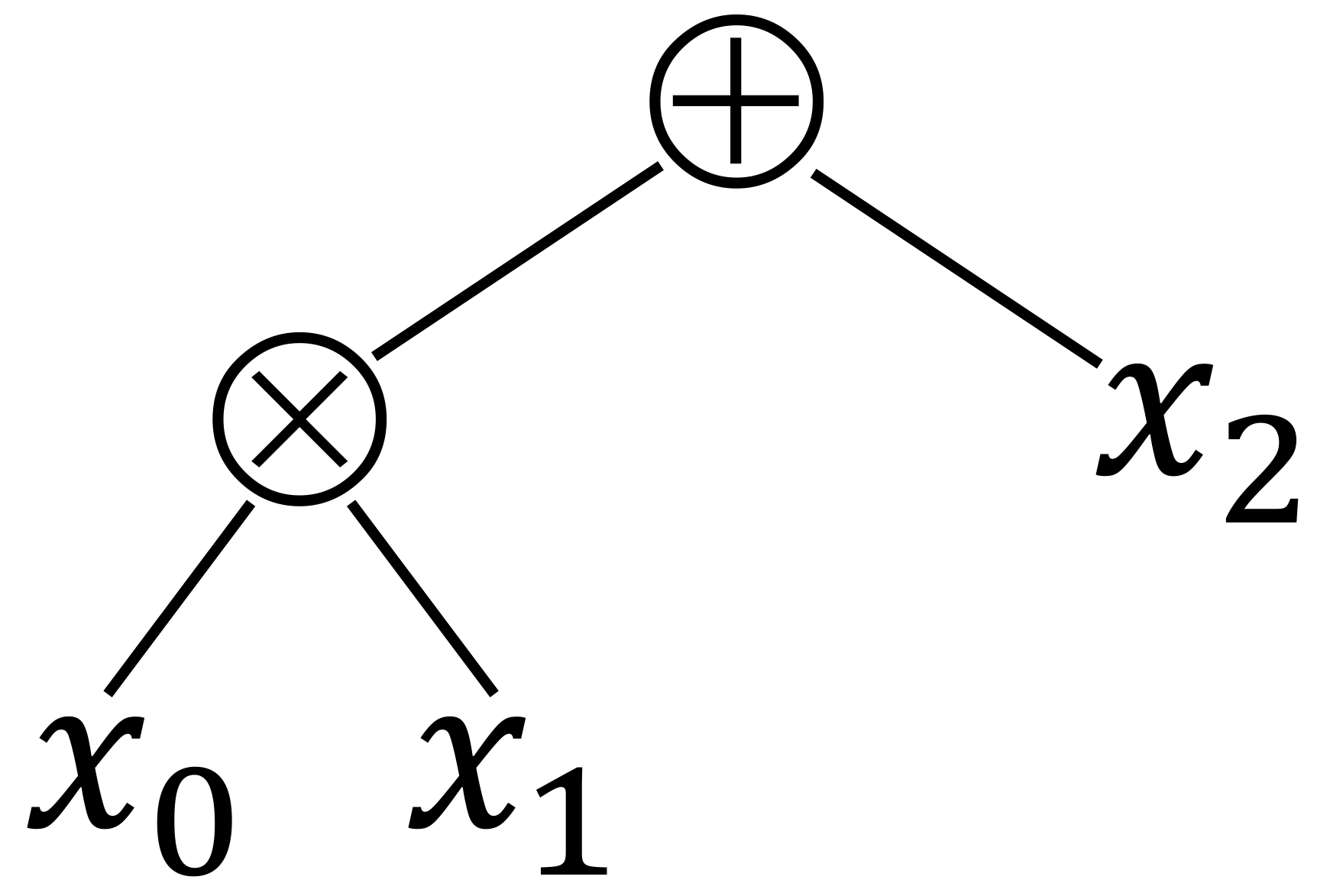}
        \caption{A circuit.} \label{fig:unsmoothed}
    \end{subfigure}
    \hspace{0.1\linewidth}
    \begin{subfigure}[b]{0.4\linewidth}
        \centering
        \includegraphics[width=\linewidth]{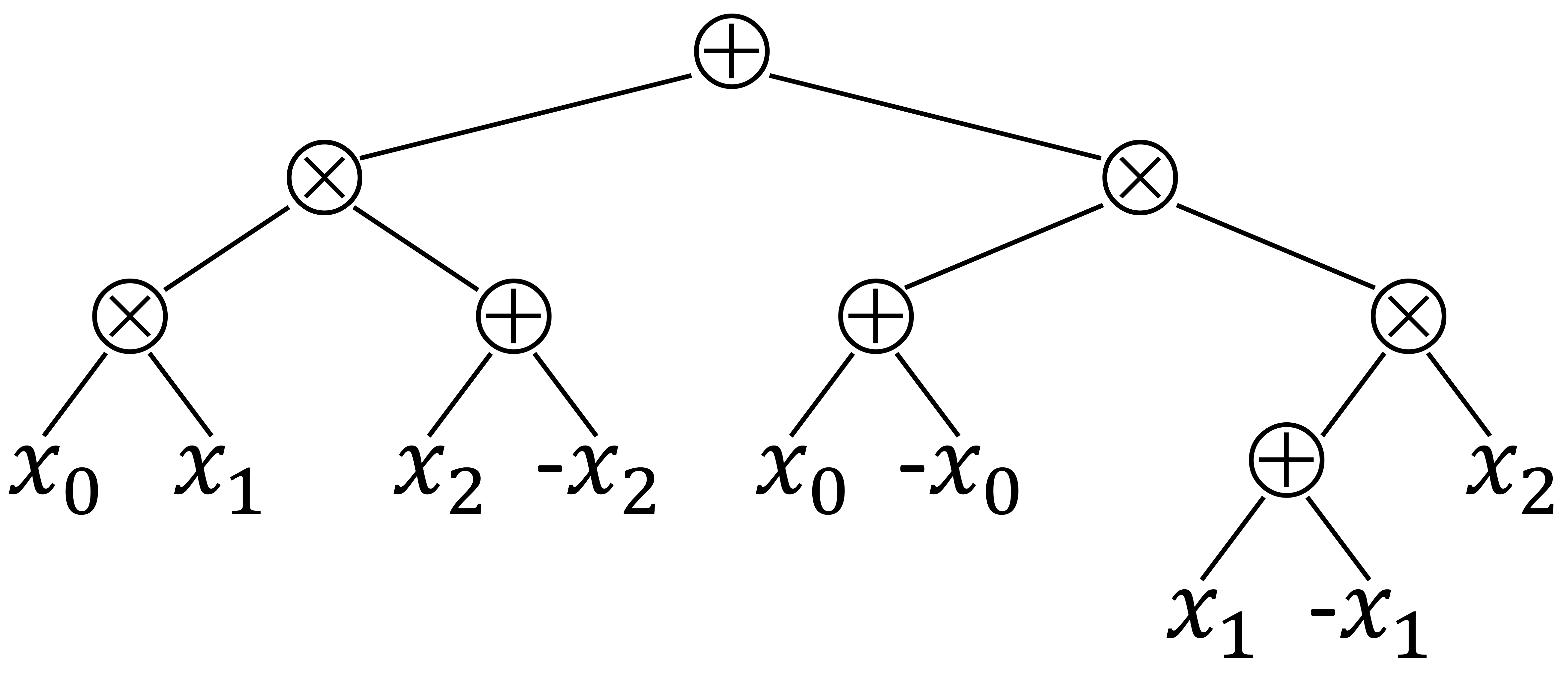}
        \caption{A smooth circuit.} \label{fig:smoothed}
    \end{subfigure}
\caption{Two equivalent circuits computing \((x_0 \otimes x_1) \oplus x_2\). The left one is not smooth and the right one is smooth.} \label{fig:ex1}
\end{figure}

In this paper we propose a more efficient smoothing algorithm. We focus on the commonly used class of \emph{structured decomposable circuits}, which include structured decomposable Negation Normal Form, Sentential Decision Diagrams, and more~\citep{Pipatsrisawat2008NewCL,Darwiche2011SDDAN}. Intuitively, structuredness requires that circuits always consider their variables in a certain way, which is formalized as a tree structure on the variables called a \emph{vtree}.

Our first contribution (Section~\ref{sec:smoothingcircuits}) is to show a near-linear time algorithm for smoothing such circuits, which is a clear improvement on the naive quadratic algorithm. Specifically, our algorithm runs in time proportional to the circuit size multiplied by the inverse Ackermann function $\alpha$ of the circuit size and number of variables\footnote{The inverse Ackermann function $\alpha$ is defined in~\cite{Tarjan1972EfficiencyOA}. As the Ackermann function grows faster than any primitive recursive function, the function \(\alpha\) grows slower than the inverse of any primitive recursive function, e.g., slower than any number of iterated logarithms of \(n\).} (Theorem~\ref{thm:smoothingub}).

Our second contribution (Section~\ref{sec:lowerbound}) is to show a lower bound of the same complexity, on smoothing decomposable circuits for the restricted class of smoothing algorithms that we call \emph{smoothing-gate algorithms} (Theorem~\ref{thm:smoothinglb}). Intuitively, smoothing-gate algorithms are those that retain the structure of the original circuit and can only make them smooth by adding new gates to cover the missing variables. This natural class corresponds to the example in Figure~\ref{fig:ex1} and our near-linear time smoothing algorithm also falls in this class. We match its complexity and show a lower bound on the performance of \emph{any} smoothing-gate algorithm, relying on known results in the field of range-sum queries.

Our third contribution (Section~\ref{sec:allmarginals}) is to focus on the probabilistic inference task of All-Marginals and to propose a novel linear time algorithm for this task which bypasses the need for smoothing, assuming that the weight function is always positive and supports all four elementary operations of \(\oplus,\ominus,\otimes,\oslash\) (Theorem~\ref{thm:allmarginal}). These results are summarized in Table~\ref{tab:summary}.

Our fourth contribution (Section~\ref{sec:structuredness}) is to study how to make a circuit smooth while preserving structuredness. We show that we cannot achieve a sub-quadratic smoothing algorithm if we impose the same vtree structure on the output circuit unless the vtree has low height (Prop.~\ref{prop:structuredbound}).

Our final contribution (Section~\ref{sec:experiments}) is to experiment on smoothing and probabilistic inference tasks. We evaluate the performance of our smoothing and of our linear time All-Marginals algorithm.

The rest of the paper is structured as follows. In Section~\ref{sec:background} we review the necessary definitions, and in Section~\ref{sec:smoothingrole} we motivate the task of smoothing in more detail. We then present each of our five contributions in order in Sections~\ref{sec:smoothingcircuits},~\ref{sec:lowerbound},~\ref{sec:allmarginals},~\ref{sec:structuredness} and~\ref{sec:experiments}. We conclude in Section~\ref{sec:conclusion}.

\begin{table}[ht]
\center
\caption{Summary of results on structured decomposable circuits. We let \(n\) be the number of variables and \(m\) be the size of the circuit. \label{tab:summary}}
\begin{tabular}{ccc}
\toprule
{\bfseries Task}
& {\bfseries Operations}
& {\bfseries Complexity}\\
\midrule
Smoothing & \(\oplus,\otimes\) & {\(O(m \cdot \alpha(m,n))\) }\\
Smoothing$^*$ & \(\oplus,\otimes\) & {\(\Omega(m \cdot \alpha(m,n)) ^*\) }\\
All-Marginal & \(\oplus,\ominus,\otimes,\oslash\) & {\(\Theta(m)\) }\\
\bottomrule
\multicolumn{3}{c}{$^*$ For \emph{smoothing-gate algorithms} on decomposable circuits.}
\end{tabular}
\end{table}
\section{Background} \label{sec:background}

Let us now define the model of circuits that we study (refer again to Figure~\ref{fig:ex1} for an example):

\begin{definition}
A \textbf{logical circuit} is a rooted directed acyclic graph where leaves are literals, and internal gates perform disjunction (\(\oplus\)-gates) or conjunction (\(\otimes\)-gates). An \textbf{arithmetic circuit} is one where leaves are numeric constants or variables, and internal gates perform addition (\(\oplus\)-gates) or multiplication (\(\otimes\)-gates). The \textbf{children} of an internal gate are the gates that feed into it.
\end{definition}

We focus on circuits that are \emph{decomposable} and more precisely that are \emph{structured}. We first define decomposability:

\begin{definition}
For any gate \(p\), we call \(vars_p\) the set of variables that appear at or below gate \(p\). A circuit is \textbf{decomposable} if these sets of variables are disjoint between the two children of every \(\otimes\)-gate. Formally, for every \(\otimes\)-gate \(p\) with children \(c_1\) and \(c_2\), we have \(vars_{c_1} \cap vars_{c_2} = \emptyset\).
\end{definition}

We then define structuredness, by introducing the notion of a \emph{vtree} on a set of variables:

\begin{definition}
A \textbf{vtree} on a set of variables \(S\) is a full binary tree whose leaves have a one-to-one correspondence with the variables in \(S\). We denote the set of variables under a vtree node \(p\) as \(u_p\).
\end{definition}

\begin{definition}
A circuit \textbf{respects} a vtree \(V\) if each of its \(\otimes\)-gate has 0 or 2 inputs, and there is a mapping \(\rho\) from its gates to \(V\) such that:
\begin{itemize}
    \item For every variable \(c\), the node \(\rho(c)\) is mapped to the leaf of \(V\) corresponding to \(c\).
    \item For every \(\oplus\)-gate \(c\) and child \(c'\) of \(c\), the node \(\rho(c')\) is \(\rho(c)\) or a descendant of \(\rho(c)\) in \(V\).
    \item For every \(\otimes\)-gate \(c\) with children \(c_1,c_2\), letting \(v_l\) and \(v_r\) be the left and right children of \(\rho(c)\), the node \(\rho(c_1)\) is \(v_l\) or a descendant of \(v_l\) and \(\rho(c_2)\) is \(v_r\) or a descendant of \(v_r\).
    \end{itemize}
    A circuit is \textbf{structured decomposable} if it respects some vtree \(V\). The circuit is then decomposable.
\end{definition}

Recall that a circuit can be preprocessed in linear time to ensure that each \(\otimes\)-gate has 0 or 2 inputs. 

Structured decomposability was introduced in the context of logical circuits, and it is also enforced in Sentential Decision Diagrams, a widely used tractable representation of Boolean functions~\citep{Darwiche2011SDDAN}. This property allows for a polytime conjoin operation and symmetric/group queries on logical circuits~\citep{Pipatsrisawat2008NewCL,Bekker2015TractableLF}. For circuits that represent distributions, structured decomposability allows multiplication of these distributions~\citep{ShenCD16}, efficient computation of the KL-divergence between two distributions~\citep{LiangXAI17}, and more. Structured decomposable circuits are also used when one wants to induce distributions over arbitrary logical formulae~\citep{Kisa2014ProbabilisticSD} or compile a logical formula bottom-up~\citep{oztok2015top}.

Next, we review another property of logical circuits that is relevant for probabilistic inference tasks~\citep{Darwiche2001OnTT,Choi2017OnRD}.

\begin{definition}
A logical circuit on variables \(\X\) is \textbf{deterministic} if under any input \(\x\), at most one child of each \(\oplus\)-gate evaluates to true.
\end{definition}

In the rest of this paper, we will let \(n\) denote the number of variables in a circuit and let \(m \geq n\) denote the size of a circuit, measured by the number of edges in the circuit.

\section{Smoothing} \label{sec:smoothingrole}

We focus on the probabilistic inference tasks of weighted model counting and computing All-Marginals~\citep{Sang2005PerformingBI,Chavira2008OnPI}. We will study weighted model counting in the more general form of \emph{Algebraic Model Counting} (AMC)~\citep{KimmigJAL16}. To describe these tasks, we define instantiations, knowledge bases and models.
\begin{definition}
Given a set of variables \(\X\), a full assignment of all the variables in \(\X\) is called an \textbf{instantiation}. A set \(f\) of instantiations is called a \textbf{knowledge base}, and each instantiation in \(f\) is called a \textbf{model}.
\end{definition}

The AMC task on a knowledge base \(f\) and a weight function \(w\) (a mapping from the literals to the reals) is to compute \(s\) from Equation~\ref{eq:amc}. The task of All-Marginals is to compute the partial derivative of \(s\) with respect to the weight of each literal as in Equation~\ref{eq:allmarginal}.

\vspace{-.3cm} 
\begin{minipage}{0.4\linewidth} 
    \center
    \begin{equation} 
        \textstyle s = \bigoplus_{\x \in f}{\bigotimes_{x \in \x} {w(x)}} \quad \text{AMC} \label{eq:amc}
    \end{equation} 
\end{minipage} 
\hspace{0.08\linewidth}
\begin{minipage}{0.5\linewidth} 
    \center
    \begin{equation} 
        \textstyle \left\{\frac{\partial s}{\partial w(x)}, \frac{\partial s}{\partial w(-x)} \, \middle| \, X \in \X \right\} \quad \text{All-Marginals} \label{eq:allmarginal}
    \end{equation} 
\end{minipage}

On probabilistic models, $s$ is often the partition function or the probability of evidence, where the partial derivatives of these quantities correspond to all (conditional) marginals in the distribution. Computing All-Marginals efficiently significantly speeds up probabilistic inference, and is used as a subroutine in the collapsed sampling algorithm in our later experiments.

These tasks are difficult in general, unless we have a tractable representation of the knowledge base~\(f\). 
Moreover, it is important to have a smooth representation.
Indeed, suppose \(f\) is represented as a logical circuit that is only deterministic and decomposable but not smooth. Then, there is in general no known technique to perform the AMC and All-Marginals tasks in linear time (although there is a special case where AMC can be performed in linear time, explained below). By contrast, if \(f\) is represented as a logical circuit that is deterministic, decomposable and smooth, then the AMC and All-Marginals tasks can be performed in time \(O(m)\). For example, the AMC task is done by converting the deterministic, decomposable and smooth logical circuit into an arithmetic circuit, attaching the weights of the variables as numeric constants in the circuit, and then evaluating the circuit. Furthermore, when a decomposable arithmetic circuit computes a factor (a mapping from instantiations to the reals), enforcing smoothness allows it to compute factor marginals in linear time~\citep{Choi2017OnRD}.

As smoothing is necessary to efficiently solve these inference tasks, we are interested in studying the complexity of smoothing a circuit. To do so, we formally define the task of smoothing.

\begin{definition}
Two logical circuits on variables \(\X\) are \textbf{equivalent} if they evaluate to the same output on any input \(\x\).
\end{definition}
\begin{definition}
A circuit is \textbf{smooth} if for every pair of children \(c_1\) and \(c_2\) of a \(\oplus\)-gate, \(vars_{c_1} = vars_{c_2}\).
\end{definition}
\begin{definition}
The task of \textbf{smoothing} a decomposable logical circuit is to output a smooth and decomposable logical circuit that is equivalent to the input circuit. Similarly, the task of \textbf{smoothing} a deterministic and decomposable logical circuit is to output a smooth, deterministic, and decomposable circuit that is equivalent to the input circuit.
\end{definition}

We only define the smoothing task over logical circuits. This is because the probabilistic inference tasks are performed by smoothing a logical circuit and then converting it into an arithmetic circuit, so it is easier for the reader to only consider smoothing on logical circuits. For the rest of the paper, we will refer to logical circuits simply as circuits.
Note that we require the output smooth circuit to preserve the same properties (decomposability/determinism) as the input circuit. Indeed, there is a trivial linear time algorithm for smoothing that breaks decomposability (i.e., simply conjoin all gates with a tautological gate that mentions all variables), but then the resulting circuit may not be useful for probabilistic inference. 
Again, we need decomposability to compute factor marginals, and we need decomposability along with determinism to compute AMC and All-Marginals.
By contrast, we do not require the output smooth circuit to be structured, because structuredness is not required to solve our tasks of AMC or All-Marginals (nevertheless, we do study structuredness in Section~\ref{sec:structuredness}).

Sometimes, when the weight function allows division, there exists a renormalization technique that can solve AMC in linear time without smoothing the initial circuit~\citep{KimmigJAL16}. However, this restriction is limiting, since even if the weight function is defined over a field, division by zero may be unavoidable~\citep{Broeck2014SkolemizationFW}. Also, the weight function may only be defined over a semiring (\(\oplus, \otimes\))~\citep{Friesen2016TheST}. In these cases, there is no known technique to bypass smoothing. Therefore, developing an efficient smoothing algorithm is an important problem, which we address next in Sections~\ref{sec:smoothingcircuits} \&~\ref{sec:lowerbound}.

On the other hand, one may still be interested in settings where all four elementary operations of \(\oplus,\ominus,\otimes,\oslash\) on the weight function are allowed. To this end, we also propose in Section~\ref{sec:allmarginals} a novel technique that solves the All-Marginals task in linear time when the weight function is positive, and when subtraction and division are allowed.

\section{Smoothing Algorithm} \label{sec:smoothingcircuits}

We present our algorithm for smoothing structured decomposable circuits, based on the semigroup range-sum literature. First, we define a class of common strategies to smooth a circuit, which encompasses both the previously-known algorithm and our new algorithm.

 The existing quadratic algorithm for smoothing a circuit goes to each \(\oplus\)-gate and inserts missing variables one by one~\citep{Darwiche2001OnTT}. This algorithm retains the original gates of the circuit, and adds additional gates to fill in missing variables. We will define \emph{smoothing-gate algorithms} as the family of smoothing algorithms that retain the original gates of the circuit.

\begin{definition}
\textbf{Edge contraction} is the process of removing each \(\oplus\)-gate or \(\otimes\)-gate with a single child, and feeding the child as input to each parent of the removed gate.
\end{definition}

\begin{definition}
A \textbf{subcircuit} of a circuit is another circuit formed by taking a subset of the gates and edges of the circuit, and picking a new root. The gate subset must include the new root and all endpoints of the edge subset.
\end{definition}

\begin{definition}
\label{def:smoothinggatealg}
Two circuits \(g\) and \(h\) with gate sets \(G\) and \(H\) are \textbf{isomorphic} if there exists a bijection \(B : G \rightarrow H\) between their gates such that the following conditions hold.
\begin{enumerate}
    \item For any gate \(p \in G\), \(B(p)\) is the same type of gate as \(p\). 
    \item For any gate \(p_1 \in G\) and child \(p_2 \in G\) of \(p_1\), the gate \(B(p_2)\) is a child of \(B(p_1)\) in \(h\).
    \item For any gate \(p_1' \in H\) and child \(p_2' \in H\) of \(p_1'\), the gate \(B^{-1}(p_2')\) is a child of \(B^{-1}(p_1')\) in \(g\).
    \item The root of \(g\) maps to the root of \(h\).
\end{enumerate}

An algorithm is a \textbf{smoothing-gate algorithm} if for any edge-contracted (deterministic and) decomposable input circuit~\(g\), the output circuit is smooth and (deterministic and)  decomposable, is equivalent to \(g\), and has a subcircuit that is isomorphic to \(g\) after edge contraction.
\end{definition}

\emph{Smoothing-gate algorithms} are intuitive, since the structure of the original circuit is preserved. This includes the quadratic algorithm, as well as algorithms which identify missing variables under each gate and attach tautological gates to fill in those missing variables, as was done in Figure~\ref{fig:ex1}. Formally:

\begin{definition} \label{def:smoothinggate}
A gate \(g\) is called a \textbf{smoothing gate} for a set of variables \(\X\) if \(vars_g = \X\) and the circuit rooted at \(g\) is tautological and decomposable. We denote such a gate by \(SG(\X)\).
\end{definition}

The structure of a smoothing gate \(SG(\X)\) is not specified. The only requirement is that it mentions all variables in \(\X\) and is tautological and decomposable. For example, the quadratic algorithm constructs each \(SG(\X)\) by naively conjoining \(x \oplus -x\) for each variable in \(\X\) one at a time, leading to a linear amount of work per gate. In the case of structured decomposable circuits, we can do much better.

\begin{lemma} \label{lem:intervalgap}
Consider a structured decomposable circuit, and let \(\pi\) be the sequence of its variables written following the in-order traversal of its vtree. For any two vtree nodes \((\rho(p), \rho(c))\), we have that \(u_{\rho(p)} \backslash u_{\rho(c)}\) can be written as the union of at most two intervals in \(\pi\).
\end{lemma}
\begin{proof}
Since \(v\) is a binary tree, the in-order traversal of \(v\) visits the variables of \(u_{\rho(p)}\) consecutively, and the variables of \(u_{\rho(c)}\) consecutively. Hence, \(u_{\rho(p)}\) and \(u_{\rho(c)}\) can each be written as one interval, and \(u_{\rho(p)} \backslash u_{\rho(c)}\) can be written as the union of at most two intervals.
\end{proof}

We then smooth a circuit in one bottom-up pass. If \(p\) is a leaf \(\otimes\)-gate, replace it with \(SG(u_{\rho(p)})\). If \(p\) is an internal \(\otimes\)-gate, letting \(v_l,v_r\) and \(c_1,c_2\) be the children of \(\rho(p)\) and \(p\) respectively, replace \(c_1\) with \(c_1 \otimes SG(u_{v_l} \backslash u_{\rho(c_1)})\) and \(c_2\) with \(c_2 \otimes SG(u_{v_r} \backslash u_{\rho(c_2)})\). If \(p\) is a \(\oplus\)-gate, replace each child \(c\) with \(c \otimes SG(u_{\rho(p)} \backslash u_{\rho(c)})\).  By Lemma~\ref{lem:intervalgap}, each smoothing gate can be built by multiplying together two gates of the form \(\bigotimes_{\X}(x \oplus -x)\), where \(\X\) forms an interval in \(\pi\). Thus, we can appeal to results from semigroup range-sums, by treating each \(x \oplus -x\) as an element in a semigroup, and treating the computation of \(\bigotimes_{\X}(x \oplus -x)\) as a ``summation'' in the semigroup over an interval (range).

\paragraph*{Semigroup Range-Sum.} \label{sec:rangesum}

The semigroup range-sum problem considers a sequence of \(n\) variables $x_1, \ldots, x_n$, a sequence of \(m \geq n\) intervals \([a_1,b_1],\ldots,[a_m,b_m]\) of these variables, and a weight function \(w\) from the variables to a semigroup. The task is to compute the sum of weights of the variables in each interval, i.e. \(s_j = \Sigma_{i \in [a_j,b_j]} w(x_i)\) for all \(j \in [1,m]\)~\citep{Yao1982SpaceTimeTF,Chazelle1989ComputingPS}. Since \(w\) is only defined over a semigroup, subtraction is not supported. That is, we cannot follow the efficient strategy of precomputing all \(p_k = \Sigma_{i \in [1,k]} w(x_i)\) and outputting \(s_j = p_{b_j} - p_{a_j-1}\). Still, there is an efficient algorithm to compute all the required sums in time \(O(m \cdot \alpha(m,n))\), where \(\alpha\) is the inverse Ackermann function. We restate their result here.

\begin{theorem} \label{thm:rangesumub}
Given \(n\) variables defined over a semigroup and \(m\) intervals, the sum of all intervals can be computed using \(O(m \cdot \alpha(m,n))\) additions~\citep{Chazelle1989ComputingPS}.
\end{theorem}

Our smoothing task can be reduced to the semigroup range-sum problem as follows. Smoothing a structured decomposable circuit of size \(m\) reduces to constructing smoothing gates for \(O(m)\) intervals. We pass these intervals as input to the range-sums algorithm, which will then generate a sequence of additions that computes the sum of each interval: each addition adds two numbers that are either individual variable weights or a sum that was previously computed.

We then trace this sequence of additions (see Figure~\ref{fig:algtrace}). For the base case of \(w(x_i)\), let \(g(w(x_i))\) be the gate \(x_i \oplus -x_i\). Then for each addition \(s = t + u\), we construct a corresponding \(\otimes\)-gate \(g(s) = g(t) \otimes g(u)\). In particular, when an addition in the sequence has computed the sum of an interval, then the corresponding gate is a smoothing gate for that interval. This process preserves determinism, so it converts a (deterministic and) structured decomposable circuit into a smooth and (deterministic and) decomposable circuit. The output circuit is generally no longer structured.

\begin{theorem} \label{thm:smoothingub}
The task of smoothing a (deterministic and) structured decomposable circuit has time complexity \(O(m \cdot \alpha(m,n))\), where \(n\) is the number of variables and \(m\) is the size of the circuit.
\end{theorem}

Although~\cite{Chazelle1989ComputingPS} do not formally assert a time complexity on determining the sequence of additions to perform, we show that there is no overhead to this step. That is,~\cite{Chazelle1989ComputingPS} show that there exists a sequence of \(O(m \cdot \alpha(m,n))\) additions, and we additionally prove that this sequence can be computed in time \(O(m \cdot \alpha(m,n))\). The proof is in the Appendix.

\begin{figure}[t]
    \begin{subfigure}[b]{0.55\linewidth}
        \begin{verbatim}
            a = w(x_2) + w(x_3)
            b = a + w(x_1)
            c = a + w(x_4)
            output b, c\end{verbatim}
        \caption{Sequence of additions to compute intervals} \label{fig:intervals}
    \end{subfigure}
    \begin{subfigure}[b]{0.38\linewidth}
        \centering
        \includegraphics[width=\linewidth]{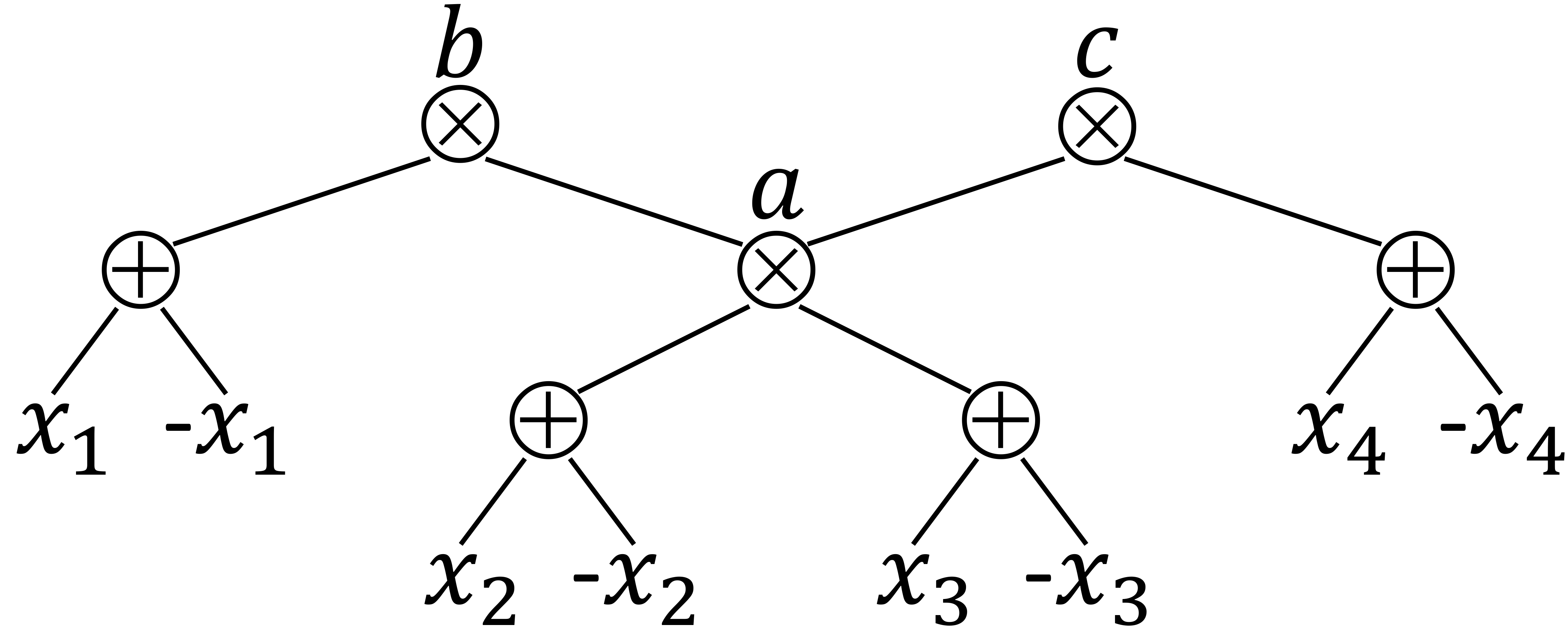}
        \caption{Tracing the additions into a circuit} \label{fig:trace}
    \end{subfigure}
\caption{We construct smoothing gates for \(\{x_1,x_2,x_3\}\) and \(\{x_2,x_3,x_4\}\) by first passing the intervals \([1,3]\) and \([2,4]\) to the range-sum algorithm, and then tracing the sequence of additions. The trace is done by replacing \(w(x_i)\) with \(x_i \oplus -x_i\) and replacing each addition with a \(\otimes\)-gate.} \label{fig:algtrace}
\end{figure}

\section{Lower Bound} \label{sec:lowerbound}

In this section we show a lower bound on the task of smoothing a decomposable circuit, for the family of smoothing-gate algorithms. First we state an existing lower bound on semigroup range-sums:

\begin{theorem} \label{thm:rangesumlb}
Given \(n\) variables defined over a semigroup, for all algorithms there exists a set of \(m=n\) intervals such that computing the sum of the weights of the variables for each interval takes \(\Omega(m \cdot \alpha(m,n))\) number of additions~\citep{Chazelle1989ComputingPS}.
\end{theorem}

We cannot immediately assert the same lower bound for the problem of smoothing decomposable circuits, for two reasons. First, we must reduce the computation of the \(m\) interval sums to a smoothing problem, and express this reduction in a circuit taking no more than \(O(m)\) space. Second, we must show that no smoothing algorithm is more efficient than smoothing-gate algorithms. We address the first issue but leave the second open, leading to the following theorem with the proof in the Appendix.

\begin{theorem} \label{thm:smoothinglb}
For smoothing-gate algorithms, the task of smoothing a decomposable circuit has space complexity \(\Omega(m \cdot \alpha(m,n))\), where \(n\) is the number of variables and \(m\) is the size of the circuit.
\end{theorem}

\section{Computing All-Marginals} \label{sec:allmarginals}

In this section, we focus on the specific task of computing All-Marginals on a knowledge base represented as a deterministic and structured decomposable circuit. Remember that the goal is to compute the partial derivative of the circuit with respect to the weight of each literal (Equation~\ref{eq:allmarginal} in Section~\ref{sec:smoothingrole}). If the input circuit is smooth, then we can solve the task in time linear in the size of the circuit. Therefore, with the techniques in Section~\ref{sec:smoothingcircuits}, given an input deterministic and structured decomposable circuit, we can smooth it and then convert it into an arithmetic circuit to compute All-Marginals, all in time \(O(m \cdot \alpha(m,n))\). In this section, we show a more efficient algorithm that bypasses smoothing altogether, when we assume that the weight function also supports division and subtraction and is always positive (so that we never divide by zero). The method that we propose takes time \(O(m)\), which is optimal and saves us the effort of modifying the input circuit.

\setlength\multicolsep{3pt}

\def\main{\texttt{\mbox{main}}}
\def\getinterval{\texttt{\mbox{getinterval}}}
\def\allmarginals{\texttt{\mbox{all-marginals}}}
\def\bottomup{\texttt{\mbox{bottom-up}}}
\def\topdown{\texttt{\mbox{top-down}}}
\begin{algorithm}[ht]
\caption{\allmarginals(\(g, w\)) \label{alg:allmarginals}}

We compute partial derivatives of positive literals. The negative literals are handled similarly.

\Ainput{A deterministic and structured decomposable circuit \(g\) on \(n\) variables and a weight function \(w\) that is always positive and supports \(\oplus,\ominus,\otimes,\oslash\).}

\Aoutput{Partial derivatives \(d_j\) for \(1 \leq j \leq n\).}

\begin{multicols}{2}

\textbf{\main(\(g, w\)):}

\begin{algorithmic}[1]
\STATE \(s \gets \bottomup(g, w)\) \COMMENT{requires \(\oplus,\otimes,\oslash\)}
\STATE \textbf{return} \(\topdown(g,w,s)\)
\end{algorithmic}
\vspace{1mm}

\textbf{\topdown(\(g, w, s\)):}

\begin{algorithmic}[1]
\STATE \(D \gets \{\text{root of g} : s[\text{root of g}]\}\) \COMMENT{cache}
\FOR{gates \(p\) in \(g\), parents before children}
    \STATE \(\algorithmicif \ p \text{ is a leaf } \algorithmicthen \ d_p \gets D[p]\)
    \IF{\(p\) is a \(\otimes\)-gate with children \(C\)}
        \STATE \(m \gets (\bigotimes_{k \in C}{s[k]}) \otimes D[p]\)
        \FOR{each child \(k\) in \(C\)}
            \STATE \(D[k] \gets D[k] \oplus (m \oslash s[k])\)   
        \ENDFOR
    \ENDIF
    \IF{\(p\) is a \(\oplus\)-gate with children \(C\)}
        \FOR{each child \(k\) in \(C\)}
            \STATE \(l_1,r_1,l_2,r_2 \gets \getinterval(p,k)\)
            \STATE \(\Delta_{l_1} \gets \Delta_{l_1} \oplus D[p]\)
            \STATE \(\Delta_{r_1+1} \gets \Delta_{r_1+1} \ominus D[p]\)
            \STATE \(\Delta_{l_2} \gets \Delta_{l_2} \oplus D[p]\)
            \STATE \(\Delta_{r_2+1} \gets \Delta_{r_2+1} \ominus D[p]\)
            \STATE \(D[k] \gets D[k] \oplus D[p]\)
        \ENDFOR
    \ENDIF
\ENDFOR

\STATE \(d_1 \gets d_1 \oplus \Delta_1\)
\STATE \(\algorithmicfor \ i \gets [2,n] \ \algorithmicdo  \ d_j \gets d_{j-1} \oplus \Delta_j\)
\STATE \textbf{return} \(d\)
\end{algorithmic}
\vspace{1mm}

\end{multicols}
\end{algorithm}

The algorithm is a form of backpropagation, and goes as follows (Algorithm~\ref{alg:allmarginals}). First, we compute the circuit output using a linear bottom-up pass over the circuit in the \(\bottomup\) subroutine, the details of which are omitted.
During this process, we keep track of the contribution of each internal gate using the dictionary \(s\). Next, we traverse the circuit top-down
in order to compute the partial derivative of each gate. At a \(\otimes\)-gate or \(\oplus\)-gate, we propagate the partial derivative down to the children as needed. However, since the circuit is not smooth, there may be missing variables in the children of \(\oplus\)-gates, in which case the propagation is incomplete. The challenge is then to efficiently complete the propagation to the missing variables. We optimize this propagation step using range increments, which gives us the next theorem with a proof in the Appendix.

\begin{theorem} \label{thm:allmarginal}
The All-Marginals task on a deterministic and structured decomposable circuit \(g\) and a weight function \(w\) that is always positive and supports \(\oplus,\ominus,\otimes,\oslash\) has time complexity \(\Theta(m)\), where \(n\) is the number of variables and \(m\) is the size of \(g\).
\end{theorem}

\section{On Retaining Structuredness} \label{sec:structuredness}

Recall that our smoothing algorithm in Theorem~\ref{thm:smoothingub} does not preserve structuredness of the input circuit, because it constructs smoothing gates in a way that is efficient but not structured. While structuredness is not required to solve problems such as AMC or all-marginals, it is still useful because it allows for a polytime conjoin operation, multiplication of distributions, and more (see Section~\ref{sec:background}). In this section, we show that we cannot match the performance of Theorem~\ref{thm:smoothingub} while retaining structuredness, because \emph{any} smoothing algorithm that maintains the same vtree structure must run in quadratic time. We leave open the question of whether there would be an efficient smoothing algorithm producing a circuit structured with a different vtree.

\begin{proposition} \label{prop:structuredbound}
The task of smoothing a (deterministic and) structured decomposable circuit \(g\) while enforcing the same vtree has space complexity \(\Theta(hm)\), where \(h\) is the height of the vtree and \(m\) is the size of \(g\).
\end{proposition}
\begin{proof}

\textbf{Upper bound:} We construct smoothing gates 
following the structure of the vtree: for each vtree node \(p\) with children \(p_l\) and \(p_r\), we build in constant time a structured smoothing gate for the variables that are descendants of \(p\), using the smoothing gate for the variables that are descendants of \(p_l\) and the one for the variables that are descendants of \(p_r\). Now, we can use these gates to smooth the circuit: any interval of variables in the in-order traversal of the vtree can be written as \(h\) intervals corresponding to vtree nodes, so smoothing \(g\) has time complexity \(O(hm)\). As with Theorem~\ref{thm:rangesumub}, the process of attaching smoothing gates preserves determinism.

\textbf{Lower bound:} Consider a right-linear vtree \(v\) with height \(h=n\) and variables \(X_1,\ldots,X_n\), in that order. For simplicity, let \(n\) be a multiple of \(3\), and consider the following functions for \(y \in [0, 2^{n/3})\):
\begin{align*}
\begin{array}{cc}
J_y = \bigotimes_{i=1}^{n/3}{\beta(i,y) x_i} \quad \quad & \quad \quad
K_y = \bigotimes_{i=2n/3+1}^{n}{\beta(i,y) x_i}
\end{array}
\end{align*}
where
\(\beta(i,y) = 1\) if the \(i\)-th bit of the binary representation of \(y\) is set, and \(-1\) otherwise.

Next, consider \(f = (\bigotimes_{i=1}^{n}{-x_i}) \oplus (\bigoplus_{y=1}^{2^{n/3-1}} (J_y \otimes K_y))\). An instantiation satisfies \(f\) if all its literals are negative, or if the sign of its literals from \(X_1,\ldots,X_{n/3}\) (in order) equals those from \(X_{2n/3+1},\ldots,X_n\), and are not all negative. The all-negative case is included so that \(f\) mentions all \(n\) variables, as otherwise \(f\) would already be smooth.
We can build a circuit \(g\) with size \(O(2^{n/3})\) that respects~\(v\) and computes \(f\)
using an Ordered Binary Decision Diagram representation~\citep{Bryant1986GraphBasedAF}. Yet, any smooth circuit \(G\) that respects \(v\) and computes \(f\) has size \(\Omega(n \cdot 2^{n/3})\), as we see next.

First, we use a standard notion on circuits which we refer to as a \emph{certificate}, following the terminology by~\cite{Bova2014ExpanderCH}. A certificate is formed by keeping exactly one child of each \(\oplus\)-gate, and keeping all children of each \(\otimes\)-gate.
Since \(G\) is smooth and decomposable, every certificate of \(G\) must have exactly \(n\) literals, and corresponds to an instantiation of the \(n\) variables. Let \(C_y\) be a certificate of \(G\) whose corresponding instantiation satisfies \(J_y \otimes K_y\), and let \(C_z\) be a certificate of \(h\) whose corresponding instantiation satisfies \(J_z \otimes K_z\), with \(y \neq z\) and \(y,z \in [1,2^{n/3})\). 

Next, let \(p_i\) denote the parent of the vtree node corresponding to variable \(X_i\). We will show that \(C_y\) and \(C_z\) cannot share a gate \(w\) which maps to vtree node \(p_k\) for \(k \in [n/3+1, 2n/3]\). Suppose that such a gate \(w\) exists. Then we can form a new certificate by swapping out the subtree of certificate \(C_y\) rooted at \(w\) with the subtree of certificate \(C_z\) rooted at \(w\). This new certificate now satisfies \(J_y \otimes K_z\) and is a valid certificate of the circuit \(G\), which contradicts the fact that \(G\) computes \(f\).

To finish, we consider \(2^{n/3}-1\) different certificates satisfying \(J_y \otimes K_y\) for \(y \in [1,2^{n/3})\). None of these certificates can share any gates that map to vtree nodes \(p_k\) for \(k \in [n/3+1, 2n/3]\). It follows that the output circuit \(G\) has size \(\Omega(n \cdot 2^{n/3})\).
Since the input circuit \(g\) is deterministic (because it is an OBDD) and the output circuit \(G\) need not be deterministic, the lower bound applies to both smoothing tasks (with and without determinism).
\end{proof}

\section{Experiments} \label{sec:experiments}

We experiment on our smoothing algorithm in Section~\ref{sec:smoothingcircuits} and our All-Marginals algorithm in Section~\ref{sec:allmarginals}.\footnote{The code for our experiments can be found at \url{https://github.com/AndyShih12/SSDC}. There are some differences in our implementation, which we explain in the repository.} Experiments were run on a single Intel(R) Core(TM) i7-3770 CPU with 16GB of RAM.

\paragraph*{Smoothing Circuits.}
We first study the smoothing task on structured decomposable circuits using our new smoothing algorithm (Section~\ref{sec:smoothingcircuits}), which we compare to the naive quadratic smoothing algorithm. We construct hand-crafted circuits for which many smoothing gates are required, each of which covers a large interval. In particular, we pick \(m\) large intervals \(I_1, ..., I_m\) and for each interval we construct the structured gate \(G_i = \bigotimes_{j \notin I_i}{x_j}\) for a balanced vtree. Then we take each \(G_i\) and feed them into one top-level \(\oplus\)-gate.
This triggers the worst-case quadratic behavior of the naive smoothing algorithm, while our new algorithm has near-linear behavior.

The speedup of our smoothing algorithm is captured in Table~\ref{tab:smooth}. The {\bfseries Size} column reports the size of the circuit. The {\bfseries Naive} column reports the time taken by the quadratic smoothing algorithm, the {\bfseries Ours} column reports the same value using our near-linear algorithm, and the {\bfseries Speedup} column reports the relative decrease in time. The values are averaged over \(4\) runs.

\paragraph*{Collapsed Sampling.}
We next benchmark our method for computing All-Marginals in Section~\ref{sec:allmarginals} on the task of collapsed sampling, which is a technique for probabilistic inference on factor graphs. The collapsed sampling algorithm performs approximate inference on factor graphs by alternating between \emph{knowledge compilation phases} and \emph{sampling phases}~\citep{FriedmanNeurIPS18}. In the sampling phase, the algorithm computes All-Marginals as a subroutine.

We replace the original quadratic All-Marginals subroutine by our linear time algorithm (Algorithm~\ref{alg:allmarginals}). The requirements for Algorithm~\ref{alg:allmarginals} are satisfied since the weight function \(w\) is defined over the reals and is always positive in the experiments by~\cite{FriedmanNeurIPS18}.
In Table~\ref{tab:samplingSeg} we report the results on the {\em Segmentation-11} network, which is a network from the 2006-2014 UAI Probabilistic Inference competitions. This particular network is a factor graph that was used to do image segmentation/classification (figure out what type of object each pixel corresponds to)~\citep{Forouzan2015ApproximateII}.
Experiments were also performed on other networks from the inference competition, such as {\em DBN-11} and {\em CSP-13} (Table~\ref{tab:samplingDBN} \&~\ref{tab:samplingCSP}). For all three networks we see a decrease in the number of \(\oplus,\ominus,\otimes,\oslash\) operations needed for each All-Marginal computation.
The {\bfseries Size} column reports the size threshold during the knowledge compilation phase. The {\bfseries Naive} column reports the number of \(\oplus,\ominus,\otimes,\oslash\) operations using the original All-Marginals subroutine, the {\bfseries Ours} column reports the same value using Algorithm~\ref{alg:allmarginals}, and the {\bfseries Impr} column reports the relative decrease in operations. The values are averaged over \(4\) runs.

\setlength{\tabcolsep}{3pt}
\begin{table}
\center
\footnotesize
\caption{Experiments on smoothing hand-crafted circuits and experiments on computing All-Marginals as part of the collapsed sampling algorithm. Sizes are reported in thousands (k).}
    \begin{subtable}{.48\linewidth}
      \centering
        \caption{Time (in seconds) taken to smooth circuits.\label{tab:smooth}}
        \begin{tabular}{rrrr}
        \toprule
        {\bfseries Size}
        & \multicolumn{1}{c}{{\bfseries Naive}}
        & \multicolumn{1}{c}{{\bfseries Ours}}
        & {\bfseries Speedup \(\times\)} \\
        \midrule
            40k & 0.82 \(\pm\) 0.01 & 0.04 \(\pm\) 0.01 & \phantom{00}21 \(\pm\) \phantom{0}1 \\
            416k & 50 \(\pm\) 0.3\phantom{0} & 0.31 \(\pm\) 0.01 & \phantom{0}161 \(\pm\) \phantom{0}6 \\
            1,620k & 293 \(\pm\) 2\phantom{.00} & 0.74 \(\pm\) 0.04 & \phantom{0}390 \(\pm\) 30\\
            8,500k & 6050 \(\pm\) 20\phantom{.0} & 4.13 \(\pm\) 0.09 & 1470 \(\pm\) 40 \\
        \bottomrule
        \end{tabular}
    \end{subtable} 
    \hfill
    \begin{subtable}{.51\linewidth}
      \centering
        \caption{Number of \(\oplus,\ominus,\otimes,\oslash\) operations to compute All-Marginals when sampling the Segmentation-11 network.\label{tab:samplingSeg}}
        \begin{tabular}{rrrr}
        \toprule
        {\bfseries Size}
        & \multicolumn{1}{c}{{\bfseries Naive}}
        & \multicolumn{1}{c}{{\bfseries Ours}}
        & {\bfseries Impr \%} \\
        \midrule
            100k & 28,494 \(\pm\) \phantom{0,}598 & 20,207 \(\pm\) \phantom{0,}411 & 29 \(\pm\) 3\\
            200k & 55,875 \(\pm\) 1,198 & 36,101 \(\pm\) 1,522 & 35 \(\pm\) 5\\
            400k & 86,886 \(\pm\) 6,330 & 56,094 \(\pm\) \phantom{0,}817 & 35 \(\pm\) 6\\
        \bottomrule
        \end{tabular}
    \end{subtable}%
    
    \begin{subtable}{.49\linewidth}
      \centering
        \caption{Number of \(\oplus,\ominus,\otimes,\oslash\) operations to compute All-Marginals when sampling the DBN-11 network.\label{tab:samplingDBN}}
        \begin{tabular}{rrrr}
        \toprule
        {\bfseries Size}
        & \multicolumn{1}{c}{{\bfseries Naive}}
        & \multicolumn{1}{c}{{\bfseries Ours}}
        & {\bfseries Impr \%} \\
        \midrule
            100k & 172,610 \(\pm\) 1,821 & 26,807 \(\pm\) 644 & 84 \(\pm\) 1\\
            200k & 344,748 \(\pm\) 3,881 & 51,864 \(\pm\) 851 & 85 \(\pm\) 1\\
            400k & 626,235 \(\pm\) 9,985 & 99,567 \(\pm\) 697 & 84 \(\pm\) 1\\
        \bottomrule
        \end{tabular}
    \end{subtable}%
    \hfill
    \begin{subtable}{.49\linewidth}
      \centering
        \caption{Number of \(\oplus,\ominus,\otimes,\oslash\) operations to compute All-Marginals when sampling the CSP-13 network.\label{tab:samplingCSP}}
        \begin{tabular}{rrrr}
        \toprule
        {\bfseries Size}
        & \multicolumn{1}{c}{{\bfseries Naive}}
        & \multicolumn{1}{c}{{\bfseries Ours}}
        & {\bfseries Impr \%} \\
        \midrule
            100k & 36,531 \(\pm\) 1,484 & 20,814 \(\pm\) \phantom{0,}619 & 43 \(\pm\) 4\\
            200k & 90,352 \(\pm\) 3,593 & 38,670 \(\pm\) 1,438 & 57 \(\pm\) 3\\
            400k & 122,208 \(\pm\) 9,971 & 55,269 \(\pm\) 1,819 & 54 \(\pm\) 6\\
        \bottomrule
        \end{tabular}
    \end{subtable}%
\end{table}

\section{Conclusion} \label{sec:conclusion}

In this paper we considered the task of smoothing a circuit. Circuits are widely used for inference algorithms for discrete probabilistic graphical models, and for discrete density estimation. The input circuits are required to be smooth for many of these probabilistic inference tasks, such as Algebraic Model Counting and All-Marginals. We provided a near-linear time smoothing algorithm for structured decomposable circuits and proved a matching lower bound within the class of smoothing-gate algorithms for decomposable circuits. We introduced a technique to compute All-Marginals in linear time without smoothing the circuit, when the weight function supports division and subtraction and is always positive. We additionally showed that smoothing a circuit while maintaining the same vtree structure cannot be sub-quadratic, unless the vtree has low height. Finally, we empirically evaluated our algorithms and showed a speedup over both the existing smoothing algorithm and the existing All-Marginals algorithm.

\paragraph{Acknowledgments}

This work is partially supported by NSF grants \#IIS-1657613, \#IIS-1633857, \#CCF-1837129, DARPA XAI grant \#N66001-17-2-4032, NEC Research, and gifts from Intel and Facebook Research. We thank Louis Jachiet for the helpful discussion of Theorem~\ref{thm:smoothingub}.

\bibliographystyle{named}
\bibliography{smoothing}

\newpage

\appendix

\section{Proof of Theorem~\ref{thm:smoothingub}}
\def\semigrouprangesum{\texttt{\mbox{semigroup-range-sum}}}
\def\preprocess{\texttt{\mbox{preprocess}}}
\def\inverseack{\texttt{\mbox{inverse-ack}}}
\def\ack{\texttt{\mbox{ack}}}
\def\solveinterval{\texttt{\mbox{solve-interval}}}

We analyze the semigroup range-sum scheme in Section 3 of~\cite{Chazelle1989ComputingPS}, and show that it can be implemented in \(O(m \cdot \alpha(m,n))\) time.

The scheme, as shown in Algorithm~\ref{alg:semigrouprangesum}, goes as follows. Let \(R(t,k)\) denote the largest value of \(n\) such that for any \(m \geq n\), there exists an algorithm that solves the range-sum problem on \(n\) variables and \(m\) intervals, using \(kn\) preprocessing additions and \(t\) additions per interval. This gives a total of \(kn+tm\) additions. We will show that \(R(t,k)\) is an Ackermann function by showing the following: if \(a = R(t,k-3)\), \(b = R(t-2,a)\), and \(n = ab\), then \(R(t,k) \geq n\).

It is worth mentioning that which constants we use (\(2\) and \(3\) in this case) do not matter. Any fixed constant will prove our claim. When we modify the algorithm later, it is enough to see that the algorithm works for \emph{some} fixed constants.

For the base case we have \(R(t,1) \geq 2\) and \(R(1,k) \geq 2\). To show the inductive step we first describe the preprocessing procedure (see \preprocess{} in Algorithm~\ref{alg:semigrouprangesum}). We split the \(n\) variables into \(b\) contiguous blocks, each of size \(a\). For each single block (lines 4 \& 5), we compute its prefix and suffix sums and store the results in \(d\) (lines 6-12). Next, we perform an \emph{inner recursion} where we preprocess each inner block of size \(a\). Then, we treat the sum of each inner block as new variables, and perform an \emph{outer recursion} where we preprocess the single outer block of size \(b\). We store the preprocessed results from the recursive calls in \(d\).

By the induction hypothesis, we know that during the preprocessing procedure, each inner block of size \(a\) takes at most \(a \cdot (k-3)\) additions, and the outer block of size \(b\) takes at most \(b \cdot a\) additions. Furthermore, computing the prefix/suffix sums for each inner block takes at most \(2n\) additions. Thus, the total number of additions for preprocessing is at most \(ba(k-3) + ba + 2n = (k-3)n + 3n = kn\). 

Now we describe how to compute the sum of an interval using our preprocessed results. We start at the topmost level of recursion. If the interval completely falls within an inner block, then we do an inner recursion without performing any additions (Figure~\ref{fig:CRschemeInner}). Otherwise, the interval straddles multiple inner blocks. Since we have computed the prefix/suffix sums for each inner block, we can shave off the edges of our interval using two additions. The remaining interval can be represented as a sum of inner blocks, which we compute by performing an outer recursion (Figure~\ref{fig:CRschemeOuter}).

\begin{figure}
    \centering
    \begin{subfigure}[b]{0.45\linewidth}
        \centering
        \includegraphics[width=\linewidth]{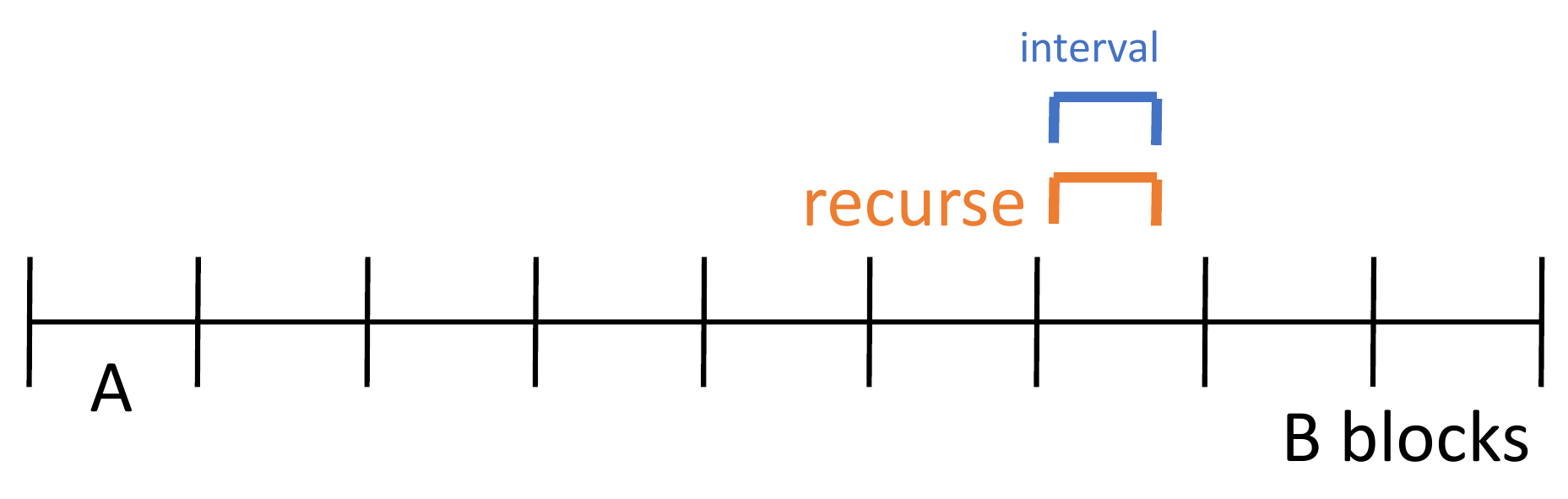}
        \caption{Inner recursion} \label{fig:CRschemeInner}
    \end{subfigure}
    \hspace{0.05\linewidth}
    \begin{subfigure}[b]{0.45\linewidth}
        \centering
        \includegraphics[width=\linewidth]{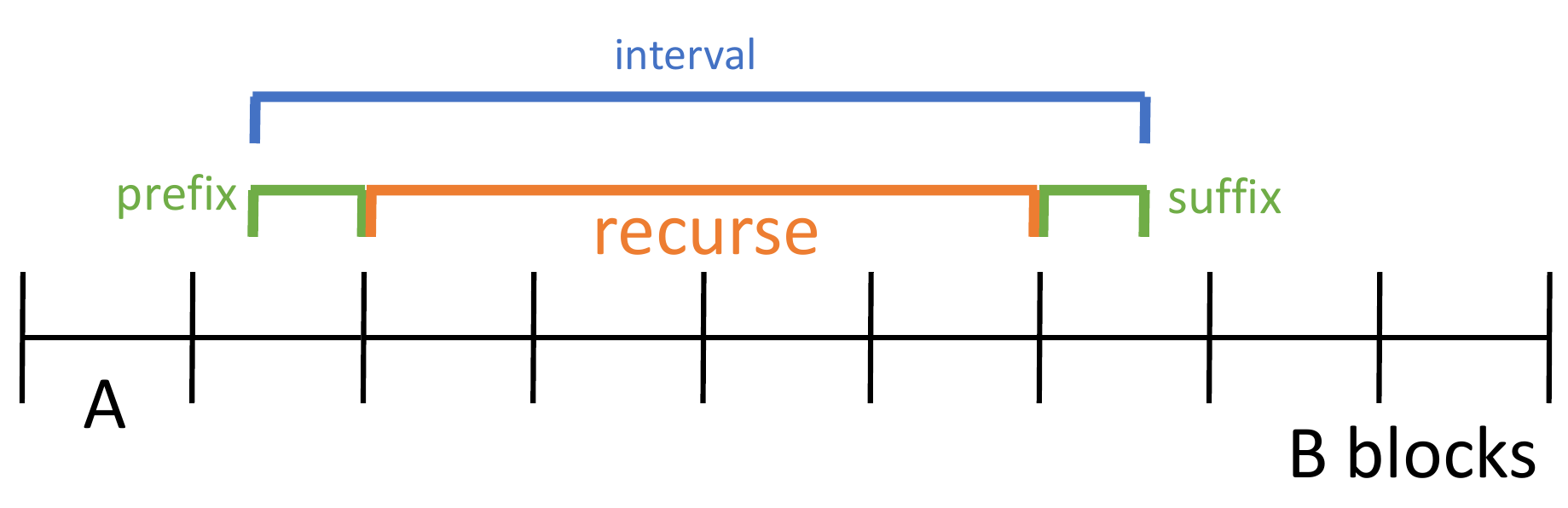}
        \caption{Outer recursion} \label{fig:CRschemeOuter}
    \end{subfigure}
    \caption{Recursive scheme for semigroup range-sum.}
    \label{fig:CRscheme}
\end{figure}

In the case of an inner recursion, we require \(t\) additions to fill the interval by induction. In the case of an outer recursion, we require \(2\) additions to shave off the edges of the interval, and the remaining sum can be computed using \(t-2\) additions by induction.


We have shown that if \(n \leq R(t,k-3)R(t-2,R(t,k-3))\), then we can solve the semigroup range-sum problem on \(n\) variables and \(m \geq n\) intervals using \(kn\) preprocessing additions and \(t\) additions per interval. Therefore \(R(t,k)\) grows as fast as an Ackermann function, so the total number of required additions is \(O(m \cdot \alpha(m,n))\).

\begin{algorithm}[ht]
\caption{\semigrouprangesum(\(I,w\)) \label{alg:semigrouprangesum}}
\Ainput{A sequence of intervals \(I = [a_1,b_1],\ldots,[a_m,b_m]\) on \(n\) variables with weights \(w = [w_0,\ldots,w_{n-1}]\) on the variables.}

\Aoutput{The sum of weights of variables for each interval.}

\textbf{\preprocess(\(\W,t,k\)):}

\begin{algorithmic}[1]
\STATE \(a \gets \ack(t,k-3) \quad\quad b \gets \ack(t-2,a)\)
\STATE \(d \gets \emptyset\)
\FOR{\(i \gets [0,b)\)}
    \STATE \(s,e \gets a \cdot i \ , \ a \cdot (i+1)\) \COMMENT{start, end of one inner block}
    \STATE \(A \gets [w_{s},...,w_{e-1}]\) \COMMENT{inner block}
    \STATE \(p_{s} \ , \ q_{e-1} \gets w_s \ , \ w_{e-1}\)
    \FOR[prefix gates: walk forward]{\(j \text{ from } s+1 \text{ to } e-1\)}
        \STATE \(p_j \gets p_{j-1} \oplus w_j\)
        \STATE \(d \gets d \cup \{p_j\}\)
    \ENDFOR
    \FOR[suffix gates: walk backward]{\(j \text{ from } e-2 \text{ down to } s\)}
        \STATE \(q_j \gets q_{j+1} \oplus w_j\)
        \STATE \(d \gets d \cup \{q_j\}\)
    \ENDFOR
    \STATE \(d \gets d \cup \preprocess(A,t,k-3)\)
    \STATE \(B_i \gets p_{e-1}\) \COMMENT{sum of inner block}
\ENDFOR
\STATE \(B \gets [B_{0},...,B_{b-1}]\) \COMMENT{outer block}
\STATE \(d \gets d \cup \preprocess(B,t-2,a)\)
\RETURN \(d\) \COMMENT{preprocessed sums}
\end{algorithmic}
\vspace{1mm}

\textbf{\main(\(I,w\)):}
\begin{algorithmic}[1]
\STATE \(n,m \gets \text{number of variables, number of intervals}\)
\STATE \(c \gets \inverseack(m,n)\)
\STATE \(d \gets \preprocess([w_0,...,w_{n-1}],2c,3c)\)
\FOR{\([a_i,b_i]\) in \(I\)}
    \STATE \(r_i \gets \solveinterval(a_i,b_i,2c,3c,d)\)
\ENDFOR
\RETURN \([r_1,\ldots,r_m]\)
\end{algorithmic}
\vspace{1mm}

\end{algorithm}

In the rest of Algorithm~\ref{alg:semigrouprangesum}, we spell out the algorithm in more detail. The subroutines~\inverseack{} and~\ack{} compute the necessary parameters, and~\preprocess{} performs the recursive preprocessing of the intervals. The subroutine~\solveinterval{} computes the sum of an interval using the recursive scheme described above and shown in Figure~\ref{fig:CRscheme}. The details of these subroutines are not provided.

It remains to show that we can implement this scheme without any extra time overhead. That is, can we \emph{find} the sequence of additions to perform using \(O(m \cdot \alpha(m,n))\) operations? The general strategy is to precompute a mapping from interval sizes to inner recursion depth, so that given an interval we can perform multiple consecutive inner recursions in one step.

The preprocessing step requires no additional overhead for finding the sequence of additions, as shown in Algorithm~\ref{alg:semigrouprangesum}; determining which addition to perform next only takes a constant amount of time (assuming we optimize with tail recursion so we do not spend a non-constant amount of time unwinding the stack). Similarly, when we need to perform an outer recursion during the processing of one interval, we only require a constant amount of time to find the two additions (prefix and suffix pieces in Figure~\ref{fig:CRschemeOuter}) and call the recursion. The problem arises when we need to perform an inner recursion. Since an inner recursion does not actually performs additions, we are not allowed any time at all to find and perform the proper recursive call. So if we perform multiple consecutive inner recursions, we will end up doing a non-constant amount of work for a single addition.

As such, we will present a technique that performs multiple consecutive inner recursions, which we will call a \emph{jump}, in a constant amount of time. After a single jump, we will either perform an outer recursion or hit a base case. In both cases, the scheme will immediately perform at least one addition, so we can absorb the (constant) cost of the jump into the addition.

\subsection{Jump Technique}

Suppose we are at some level of outer recursions given by some value \(t\). When we perform one inner recursion, we go from level \(R(t,k)\) to level \(R(t,k-c)\) for some constant \(c\). When we do a jump, we need to go from level \(R(t,k)\) to level \(R(t,k-cj)\) for some \(j \geq 1\). The details of the jumping technique are then as follows. During preprocessing, for each value of \(t\) we record the sequence of block sizes \(a^t_1 = R(t,k-c) \geq  a^t_2 = R(t,k-2c) \geq \ldots \geq a^t_{\lfloor k/c \rfloor} = R(t,k- \lfloor k/c \rfloor c)\). Then, for each value \(s \in [1,R(t,k)]\), we compute the smallest index \(i\) such that the block size \(a^t_i\) is \(\leq s\). We denote this computed index \(i\) by \(p^t_s\). This step can be done in time \(O(R(t,k) + k)\): by considering the values \(s \in [1,R(t,k)]\) in decreasing order, the indices \(p^t_s\) must be increasing in that order. So, we can compute all the indices with a two-pointer walk with cost \(O(R(t,k) + k)\), which is negligible since it is less than the original cost of precomputing the prefix and suffix sums at all the inner recursion levels for this outer recursion level (i.e., all choices of the value \(k\) for this value of \(t\)).

Given an interval of size \(s\) at outer recursion level \(t\), we can immediately find the value \(p^t_s\) and \(p^t_s - 1\). We claim that it suffices to look at inner recursion levels \(p^t_s\) and \(p^t_s - 1\). Let \(e_0 = R(t,k- p^t_s c)\) and \(e_1 = R(t,k-(p^t_s - 1)c)\). By definition, we have that \(e_0 \leq s \leq e_1\). 

\begin{itemize}
    \item In the case that the interval falls completely within one block of size \(e_1\), we will perform an outer recursion over level \(p^t_s\). We can visualize this scenario with Figure~\ref{fig:CRschemeOuter}, where \(A\) corresponds to \(e_0\) and \(AB\) corresponds to \(e_1\).
    \item Otherwise, the interval straddles exactly two blocks of size \(e_1\) at the previous inner recursion level (no more than two since \(e_1 \geq s\)). We can express the interval as a summation of a \emph{suffix sum over the first block} and a \emph{prefix sum over the second block}.
\end{itemize}

In both of the above scenarios, we skip the work of performing many inner recursive calls, and jump directly to either an outer recursive call or to a base case. So for each addition, we only do a constant amount of work, and the time complexity of solving the range-sum problem on \(n\) variables and \(m\) intervals is \(O(m \cdot \alpha(m,n))\).

\subsection{Padding}

There remains one complication: the last block in a call to the \preprocess{} function may have a different size from the rest of the blocks. For example, in Figure~\ref{fig:CRscheme}, if \(n\) is not a multiple of \(A\), then the last block will have size less than \(A\). To fix this without complicating the preprocessing algorithm, we simply pad the last block so that it is the same size as all other blocks. We also make sure that when we do an outer recursion, we only pad the original blocks as opposed to padding the padded blocks from the previous recursion level. This detail ensures that the cost of padding does not compound over multiple outer recursions. Altogether, the padding technique at most doubles the memory cost of the entire algorithm.

\section{Proof of Theorem~\ref{thm:smoothinglb}}
Take any set of \(m\) intervals, with \(m=n\). For simplicity we will let the \(m\) intervals be on \(n-2\) variables (in the range \([2,n-1]\) instead of \([1,n]\), by increasing \(n\) by \(2\) and shifting all intervals one step to the right) so that the intervals do not touch the endpoints.

First we construct prefix gates \(p_1 = x_1\) and \(p_k = p_{k-1} \otimes x_k, \forall k > 1\) in a chain-like fashion, and suffix gates \(s_n = x_n\), \(s_k = s_{k+1} \otimes x_k, \forall k < n\) in a chain-like fashion. Then for each interval \([a_i,b_i]\), construct the gate \(G_i = p_{a_i-1} \otimes s_{b_i+1}\). Next, let \(g\) be the circuit \(G_1 \oplus \ldots \oplus G_m \oplus p_{n} \oplus s_1\) (see Figure~\ref{fig:thm2circuit1}). We are attaching the \(p_n\) and \(s_1\) gate to ensure that \(g\) mentions all \(n\) variables, and that the smoothing-gate algorithm retains all prefix/suffix gates.

Since \(g\) mentions all \(n\) variables, each gate \(G_i\) also needs to mention all \(n\) variables to satisfy the smoothness property. By the construction of \(G_i\), it is missing exactly the variables \(\X_i = [X_{a_i},\ldots,X_{b_i}]\). We will show that running a smoothing-gate algorithm on \(g\) implicitly solves the semigroup range-sum problem on those intervals, by mapping the summation operation in the semigroup range-sum problem to the \(\otimes\)-gates in our circuits. Note that the circuit \(g\) is indeed decomposable and edge-contracted.

Consider a smooth and decomposable circuit \(h\) that is the output of running a smoothing-gate algorithm on \(g\). By Definition~\ref{def:smoothinggatealg}, there exists a bijection \(B\) from \(g\) to a subcircuit of \(h\) after edge-contraction. Let \(S\) denote the graph of this subcircuit before edge-contraction: we call \(S\) the \emph{skeleton graph} (see Figure~\ref{fig:thm2circuit2}). We make the two following observations.

First, we consider the gates \(B(G_i)\) for all \(i\). In the skeleton graph \(S\), there exists a path from \(B(g)\) to \(B(G_i)\), a path from \(B(G_i)\) to \(B(p_{a_i-1})\) and a path from \(B(G_i)\) to \(B(s_{b_i+1})\). We denote the set of gates on these paths (excluding the endpoints) over all \(i\) as \(T\). We observe that a gate in \(T\) must have exactly one child in \(S\), otherwise \(S\) cannot be edge-contracted into a circuit that is isomorphic to \(g\).

Since each \(\oplus\)-gate in \(T\) has exactly one child that is in the skeleton graph \(S\), we can modify \(h\) by disconnecting all other children (which do not belong to \(S\)) from these \(\oplus\)-gates, and edge-contracting these \(\oplus\)-gates. We note that this operation preserves smoothness and decomposability of \(h\), so each child of \(B(g)\) still mentions all \(n\) variables.

Second, we observe that the gate \(B(p_k)\) for any \(k\) cannot mention a variable outside of the range \([1,k]\). Otherwise, the circuit rooted at \(B(p_n)\) would implicitly contain a \(\otimes\)-gate that multiplies that variable with itself, thus violating the decomposability property. A similar argument applies to the gates \(B(s_k)\): they cannot mention a variable outside of the range \([k,n]\).

Let \(G^{\prime}_i\) denote the (unique) child of \(B(g)\) that is an ancestor of \(B(G_i)\). Recall that for any \(i\), \(G^{\prime}_i\) has the gates \(B(p_{a_i-1})\) and \(B(s_{b_i+1})\) as descendants. Furthermore, \(G^{\prime}_i\) does not have any other gate in \(\{B(p_j) : \forall j\} \cup \{B(s_j): \forall j\}\) as a descendant, otherwise it would either multiply two copies of variable \(1\) or multiply two copies of variable \(n\), and violate the decomposability property.
We now remove the set of edges in \(h\) that goes from some gate in \(T \cup \{B(G_j) : \forall j\}\) to some gate in \(\{B(p_j) : \forall j\} \cup \{B(s_j): \forall j\}\). By the above observations, the gate \(G^{\prime}_i\) must now mention exactly the variables \([a_i,b_i]\). See the transition from Figure~\ref{fig:thm2circuit2} to Figure~\ref{fig:thm2circuit3} as an example.

We now show how to extract the variables in each interval \([a,b]\) using the following relabelling scheme to remove all remaining \(\oplus\)-gates. First we remove all edges leading into \(G^{\prime}_1,\ldots,G^{\prime}_m\). Each of these \(m\) gates is still decomposable and smooth for the set of variables in its respective interval. Then for every \(\oplus\)-gate \(p\) in the circuit, take one of its input wires and reroute a copy of it to each gate that \(p\) feeds into. Each remaining \(\otimes\)-gate is now the product of one literal for each variable that was mentioned by its corresponding gate in the original circuit. These variables may be positive or negative literals, but we do not care about the polarity. We only need, for example, that if a \(\otimes\)-gate mentioned variables \(X_1,X_3,X_5\), then it is now a product of a literal of \(X_1\), a literal of \(X_3\), and a literal of \(X_5\).

After this operation, \(G^{\prime}_i\) is now exactly the product of all the variables in \([a_i,b_i]\). By setting the inputs to the circuits to be the value of the weights in the range-sum problem, and evaluating the circuits treating \(\otimes\) as addition, the value to which each gate \(G^{\prime}_i\) evaluates is the requested sum for the \(i\)-th interval. So, the circuit describes a sequence of additions to compute the sum of each interval.
We then apply Theorem~\ref{thm:rangesumlb}, which implies that the bound of \(\Omega(m \cdot \alpha(m,n))\) applies to the size of the output circuit \(h\).

\tikzstyle{vertex}=[circle,minimum size=20pt,inner sep=0pt]
\tikzstyle{edge} = [draw,thick,->]

\newcommand\myscale{1}

\begin{figure}
\center
\begin{tikzpicture}[scale=\myscale, auto,swap]
    
    \pgfmathsetmacro{\branchOneX}{-1.5}
    \pgfmathsetmacro{\branchTwoX}{0}
    \pgfmathsetmacro{\branchThreeX}{1.5}
    \pgfmathsetmacro{\branchY}{-3}
    \pgfmathsetmacro{\pRootX}{-4}
    \pgfmathsetmacro{\pRootY}{-4}
    \pgfmathsetmacro{\sRootX}{4}
    \pgfmathsetmacro{\sRootY}{-4}
    
    \foreach \pos/\id/\name in {
    {(0,0)/g/g},
    {(\pRootX,\pRootY)/pn/p_n}, {(\pRootX+1,\pRootY-1)/p3/p_3}, {(\pRootX+2,\pRootY-2)/p2/p_2}, {(\pRootX+3,\pRootY-3)/p1/p_1},
    {(\pRootX-1,\pRootY-1)/pxn/x_n}, {(\pRootX,\pRootY-2)/px3/x_3}, {(\pRootX+1,\pRootY-3)/px2/x_2}, {(\pRootX+2,\pRootY-4)/px1/x_1},
    {(\sRootX,\sRootY)/s1/s_1}, {(\sRootX-1,\sRootY-1)/s2/s_2}, {(\sRootX-2,\sRootY-2)/s3/s_3}, {(\sRootX-3,\sRootY-3)/sn/s_n},
    {(\sRootX+1,\sRootY-1)/sx1/x_1}, {(\sRootX,\sRootY-2)/sx2/x_2}, {(\sRootX-1,\sRootY-3)/sx3/x_3}, {(\sRootX-2,\sRootY-4)/sxn/x_n},
    {(\branchOneX,\branchY)/G1/G_1},{(\branchTwoX,\branchY)/G2/G_2},{(\branchThreeX,\branchY)/Gm/G_m}}
        \node[vertex] (\id) at \pos {$\name$};
        
    \foreach \source/ \dest in {
    pn/p3,p3/p2,p2/p1,
    pn/pxn,p3/px3,p2/px2,p1/px1,
    s1/s2,s2/s3,s3/sn,
    s1/sx1,s2/sx2,s3/sx3,sn/sxn,
    G1/p2,G1/sn,
    G2/p1,G2/sn,
    Gm/p1,Gm/s3}
        \path[edge] (\source) -- (\dest);
        
    \foreach \source/ \dest in {
    g/G1,g/G2,g/Gm,g/pn,g/s1}
        \path[edge,dashed] (\source) -- (\dest);
        
\end{tikzpicture}
\caption{A decomposable circuit \(g\) constructed based on input intervals. Edges are solid if the parent is a \(\otimes\)-gate, and edges are dashed if the parent is a \(\oplus\)-gate. In this example the input intervals are \(\{[3,3],[2,3],[2,2]\}\). \label{fig:thm2circuit1}}
\end{figure}
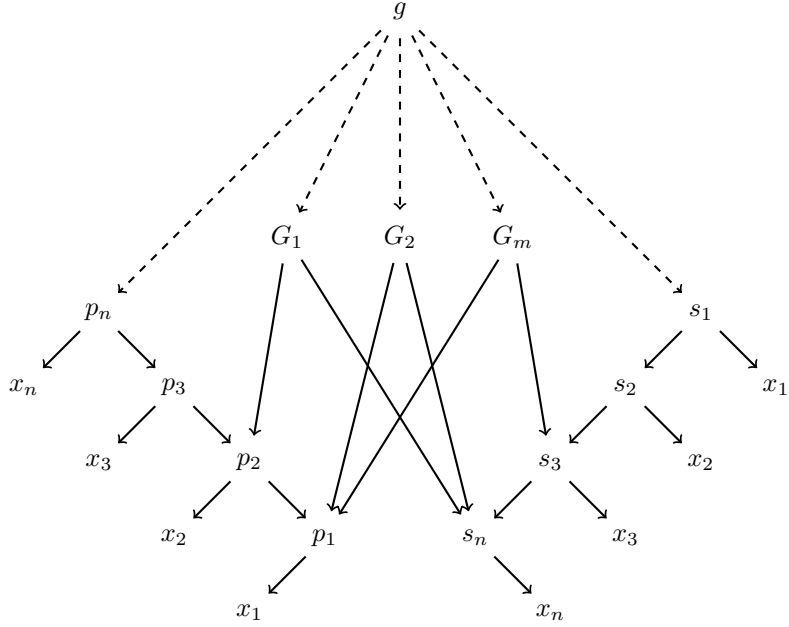


\tikzstyle{vertex}=[circle,font=\footnotesize,inner sep=0pt]
\tikzstyle{selected vertex} = [vertex, fill=blue!20]
\tikzstyle{edge} = [draw,thick,->]
\tikzstyle{selected edge} = [draw,line width=1pt,->,red]

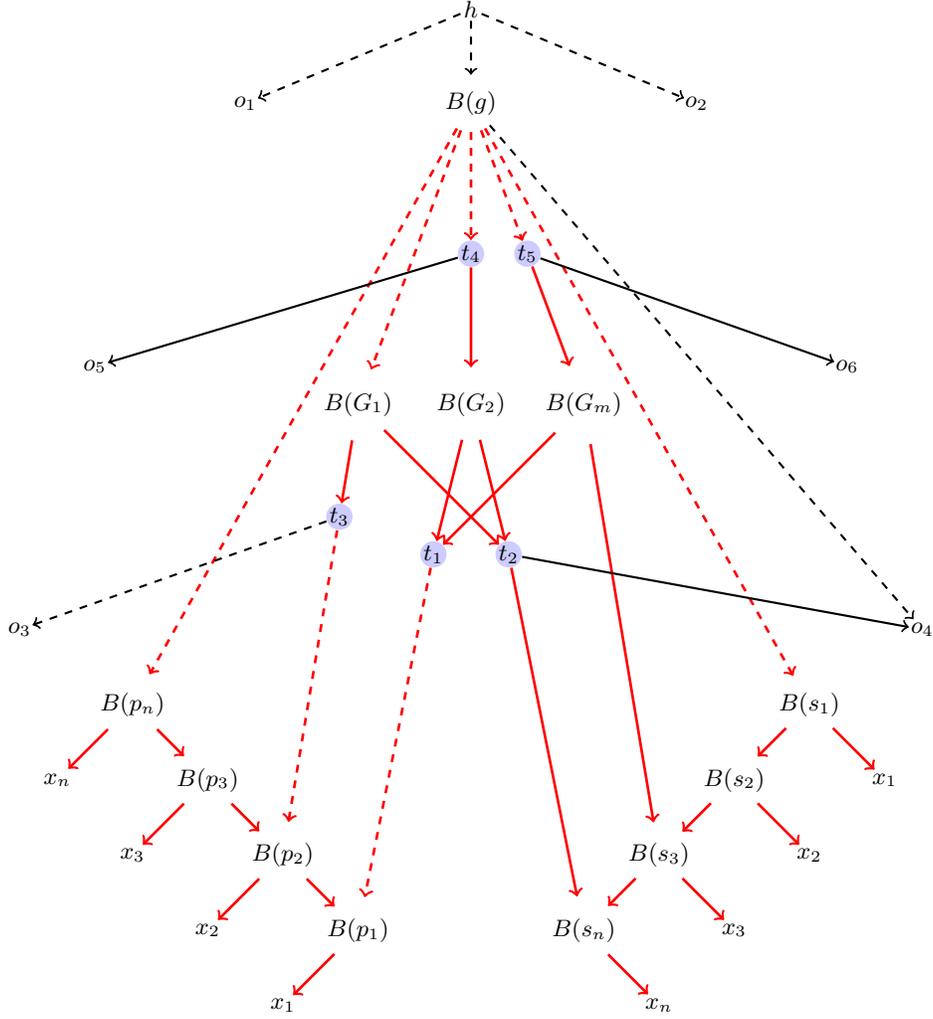
\begin{figure}
\center
\begin{tikzpicture}[scale=\myscale, auto,swap]
    
    \pgfmathsetmacro{\branchOneX}{-1.5}
    \pgfmathsetmacro{\branchTwoX}{0}
    \pgfmathsetmacro{\branchThreeX}{1.5}
    \pgfmathsetmacro{\branchY}{-4}
    \pgfmathsetmacro{\pRootX}{-4.5}
    \pgfmathsetmacro{\pRootY}{-8}
    \pgfmathsetmacro{\sRootX}{4.5}
    \pgfmathsetmacro{\sRootY}{-8}
    
    \pgfmathsetmacro{\ty}{-6}
    \pgfmathsetmacro{\txone}{- 1 / 2}
    \pgfmathsetmacro{\txtwo}{1 / 2}
    \pgfmathsetmacro{\txthree}{\branchOneX - 1/4}
    \pgfmathsetmacro{\tythree}{\ty + 0.5}
    \pgfmathsetmacro{\txfour}{\branchTwoX / 2}
    \pgfmathsetmacro{\tyfour}{\branchY / 2}
    \pgfmathsetmacro{\txfive}{\branchThreeX / 2}
    \pgfmathsetmacro{\tyfive}{\branchY / 2}

    \foreach \pos/\id/\name in {
    {(0,0)/g/B(g)},
    {(\pRootX,\pRootY)/pn/B(p_n)}, {(\pRootX+1,\pRootY-1)/p3/B(p_3)}, {(\pRootX+2,\pRootY-2)/p2/B(p_2)}, {(\pRootX+3,\pRootY-3)/p1/B(p_1)},
    {(\pRootX-1,\pRootY-1)/pxn/x_n}, {(\pRootX,\pRootY-2)/px3/x_3}, {(\pRootX+1,\pRootY-3)/px2/x_2}, {(\pRootX+2,\pRootY-4)/px1/x_1},
    {(\sRootX,\sRootY)/s1/B(s_1)}, {(\sRootX-1,\sRootY-1)/s2/B(s_2)}, {(\sRootX-2,\sRootY-2)/s3/B(s_3)}, {(\sRootX-3,\sRootY-3)/sn/B(s_n)},
    {(\sRootX+1,\sRootY-1)/sx1/x_1}, {(\sRootX,\sRootY-2)/sx2/x_2}, {(\sRootX-1,\sRootY-3)/sx3/x_3}, {(\sRootX-2,\sRootY-4)/sxn/x_n},
    {(\branchOneX,\branchY)/G1/B(G_1)},{(\branchTwoX,\branchY)/G2/B(G_2)},{(\branchThreeX,\branchY)/Gm/B(G_m)},
    {(\txone,\ty)/t1/t_1},{(\txtwo,\ty)/t2/t_2},{(\txthree,\tythree)/t3/t_3},{(\txfour,\tyfour)/t4/t_4},{(\txfive,\tyfive)/t5/t_5},
    {(0,1.25)/h/h}, {(-3,0)/o1/o_1}, {(3,0)/o2/o_2}, {(-6,-7)/o3/o_3}, {(6,-7)/o4/o_4},{(-5,-3.5)/o5/o_5},{(5,-3.5)/o6/o_6}}
        \node[vertex] (\id) at \pos {$\name$};
        
    \foreach \source/ \dest in {
    pn/p3,p3/p2,p2/p1,
    pn/pxn,p3/px3,p2/px2,p1/px1,
    s1/s2,s2/s3,s3/sn,
    s1/sx1,s2/sx2,s3/sx3,sn/sxn,
    G1/t2,
    G2/t1,G2/t2,
    Gm/t1,Gm/s3,
    G1/t3,
    t4/G2,t5/Gm,
    t2/sn}
        \path[selected edge] (\source) -- (\dest);
        
    \foreach \source/ \dest in {
    g/t4,g/G1,g/t5,g/pn,g/s1,t3/p2,t1/p1}
        \path[selected edge,dashed] (\source) -- (\dest);
        
    \foreach \source/ \dest in {
    t4/o5,t5/o6,t2/o4}
        \path[edge,] (\source) -- (\dest);

    \foreach \source/ \dest in {
    h/g,h/o1,h/o2,g/o4,t3/o3}
        \path[edge,dashed] (\source) -- (\dest);
    
    \foreach \id/\name in {t1/t_1,t2/t_2,t3/t_3,t4/t_4,t5/t_5}
        \path node[selected vertex] at (\id) {$\name$};

\end{tikzpicture}
\caption{A smooth and decomposable circuit \(h\) that is the output of a smoothing-gate algorithm on \(g\). The skeleton graph \(S\) is shown in red, and the set of gates \(T\) are circled in blue. We proceed by taking each \(\oplus\)-gate in \(T\) and removing their edges to children that are not in \(S\). This removes the edge \((t_3,o_3)\). Next, we remove the set of edges that goes from some gate in \(T \cup \{B(G_j) : \forall j\}\) to some gate in \(\{B(p_j) : \forall j\} \cup \{B(s_j): \forall j\}\). This removes the edges \((t_1,B(p_1)), (t_2,B(s_n)), (t_3, B(p_2)), (B(G_m),B(s_3))\). The gates \(t_1\) and \(t_3\) have no more children, so we prune them away. After this process, we get the circuit shown in Figure~\ref{fig:thm2circuit3}. \label{fig:thm2circuit2}}
\end{figure}


\tikzstyle{vertex}=[circle,font=\footnotesize,inner sep=0pt]
\tikzstyle{selected vertex} = [vertex, fill=blue!20]
\tikzstyle{edge} = [draw,thick,->,black!30]
\tikzstyle{selected edge} = [draw,line width=1pt,->,black]

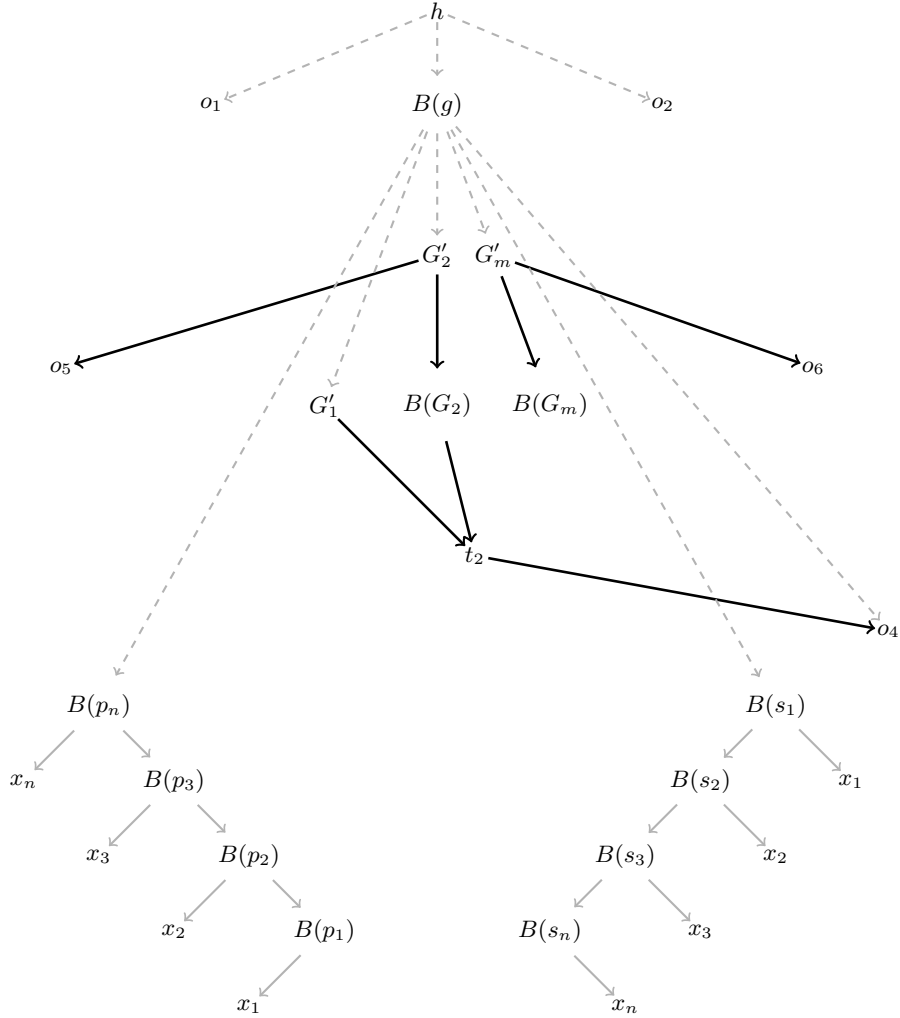
\begin{figure}
\center
\begin{tikzpicture}[scale=\myscale, auto,swap]

    \pgfmathsetmacro{\branchOneX}{-1.5}
    \pgfmathsetmacro{\branchTwoX}{0}
    \pgfmathsetmacro{\branchThreeX}{1.5}
    \pgfmathsetmacro{\branchY}{-4}
    \pgfmathsetmacro{\pRootX}{-4.5}
    \pgfmathsetmacro{\pRootY}{-8}
    \pgfmathsetmacro{\sRootX}{4.5}
    \pgfmathsetmacro{\sRootY}{-8}
    
    \pgfmathsetmacro{\ty}{-6}
    \pgfmathsetmacro{\txone}{- 1 / 2}
    \pgfmathsetmacro{\txtwo}{1 / 2}
    \pgfmathsetmacro{\txthree}{\branchOneX - 1/4}
    
    \pgfmathsetmacro{\tythree}{\ty + 0.5}
    \pgfmathsetmacro{\txfour}{\branchTwoX / 2}
    \pgfmathsetmacro{\tyfour}{\branchY / 2}
    \pgfmathsetmacro{\txfive}{\branchThreeX / 2}
    \pgfmathsetmacro{\tyfive}{\branchY / 2}

    \foreach \pos/\id/\name in {
    {(0,0)/g/B(g)},
    {(\pRootX,\pRootY)/pn/B(p_n)}, {(\pRootX+1,\pRootY-1)/p3/B(p_3)}, {(\pRootX+2,\pRootY-2)/p2/B(p_2)}, {(\pRootX+3,\pRootY-3)/p1/B(p_1)},
    {(\pRootX-1,\pRootY-1)/pxn/x_n}, {(\pRootX,\pRootY-2)/px3/x_3}, {(\pRootX+1,\pRootY-3)/px2/x_2}, {(\pRootX+2,\pRootY-4)/px1/x_1},
    {(\sRootX,\sRootY)/s1/B(s_1)}, {(\sRootX-1,\sRootY-1)/s2/B(s_2)}, {(\sRootX-2,\sRootY-2)/s3/B(s_3)}, {(\sRootX-3,\sRootY-3)/sn/B(s_n)},
    {(\sRootX+1,\sRootY-1)/sx1/x_1}, {(\sRootX,\sRootY-2)/sx2/x_2}, {(\sRootX-1,\sRootY-3)/sx3/x_3}, {(\sRootX-2,\sRootY-4)/sxn/x_n},
    {(\branchOneX,\branchY)/G1/G^\prime_1},{(\branchTwoX,\branchY)/G2/B(G_2)},{(\branchThreeX,\branchY)/Gm/B(G_m)},
    {(\txtwo,\ty)/t2/t_2},{(\txfour,\tyfour)/t4/G^\prime_2},{(\txfive,\tyfive)/t5/G^\prime_m},
    {(0,1.25)/h/h}, {(-3,0)/o1/o_1}, {(3,0)/o2/o_2}, {(6,-7)/o4/o_4},{(-5,-3.5)/o5/o_5},{(5,-3.5)/o6/o_6}}
        \node[vertex] (\id) at \pos {$\name$};
        
    \foreach \source/ \dest in {
    t4/o5,t5/o6,t2/o4,G1/t2,G2/t2,t4/G2,t5/Gm}
        \path[selected edge] (\source) -- (\dest);
        
    \foreach \source/ \dest in {
    g/pn,g/s1,h/g,h/o1,h/o2,g/o4,
    g/G1,g/t4,g/t5}
        \path[edge,dashed] (\source) -- (\dest);
        
    \foreach \source/ \dest in {
    pn/p3,p3/p2,p2/p1,
    pn/pxn,p3/px3,p2/px2,p1/px1,
    s1/s2,s2/s3,s3/sn,
    s1/sx1,s2/sx2,s3/sx3,sn/sxn}
        \path[edge] (\source) -- (\dest);

\end{tikzpicture}
\caption{The output circuit \(h\) implicitly contains a scheme for obtaining the sum of every input interval, thereby solving the semigroup range-sum problem using \(O(|h|)\) additions.\label{fig:thm2circuit3}}
\end{figure}

\section{Proof of Theorem~\ref{thm:allmarginal}}
Recall from Lemma~\ref{lem:intervalgap} that the set of missing variables of each parent-child pair forms at most two intervals with respect to the in-order traversal of the vtree. The idea now is that propagating the partial derivative to each interval amounts to a \emph{range increment}, i.e., increasing each variable in the interval by a constant. The naive algorithm takes quadratic time to do this for all intervals, but we can perform all range increments in linear time~\citep{RangeIncrementGFG}.

Consider an integer \(n\), a set of \(m\) intervals \([a_1,b_1],\ldots,[a_m,b_m]\) (\(1 \leq a_i \leq b_i \leq n\)), and \(m\) numeric constants \(c_1,\ldots,c_m\). For each integer \(1 \leq j \leq n\), we wish to compute the sum \(s_j = \bigoplus_{i : j \in [a_i,b_i]}{c_i}\). That is, if \(j\) belongs to some interval \([a_i,b_i]\), then we increase \(s_j\) by \(c_i\).
The trick is to keep track of delta variables \(\delta_1,\ldots,\delta_n\). For each interval \([a_i,b_i]\), we increase \(\delta_{a_i}\) by \(c_i\) and decrease \(\delta_{b_{i+1}}\) by \(c_i\). Finally, we output \(s_1 = \delta_1\) and \(s_j = s_{j-1} \oplus \delta_{j}, j > 1\). This process, which corresponds to Lines 11-14 and 16-17 in the \(\topdown{}\) subroutine of Algorithm~\ref{alg:allmarginals}, can be done in time \(O(m)\).

\end{document}